\documentclass{article}
\usepackage[a4paper, total={6in, 8in}]{geometry}

\usepackage{tikz}
\usepackage{amsmath}
\usepackage{amsthm}
\usepackage{hyperref, cleveref}
\hypersetup{
    colorlinks=true,
    linkcolor=blue,
    filecolor=blue,      
    urlcolor=blue,
    citecolor=blue,
    pdfpagemode=FullScreen,
    }
\usepackage{amsfonts}
\usepackage{xcolor}
\usepackage[shortlabels]{enumitem}

\usepackage[utf8]{inputenc}
\usepackage{indentfirst}
\usepackage{amsbsy,amssymb,amsthm,amsmath, amsfonts,fixltx2e,bbding} 
\usepackage{hyperref, cleveref}
\usepackage{graphicx}
\usepackage{ragged2e}

\newtheorem{thm}{Theorem}
\newtheorem{prop}[thm]{Proposition}
\newtheorem{cor}[thm]{Corollary}
\newtheorem{lem}[thm]{Lemma}

\newcommand{\ldim}{\mathsf{Ldim}}
\newcommand{\eldim}{\mathsf{effLdim}}
\newcommand{\tdim}{\mathsf{Tdim}}
\newcommand{\etdim}{\mathsf{effTdim}}

\usepackage[ruled, lined, linesnumbered, longend]{algorithm2e}

\title{Effective Littlestone Dimension}

\author{Valentino Delle Rose$^{12}$, Alexander Kozachinskiy$^{2}$, Tomasz Steifer$^{234}$}

\date{%
    $^1$Department of Computer, Control and Management Engineering “Antonio Ruberti”, Sapienza University of Rome\\%
    $^2$Centro Nacional de Inteligencia Artificial, Santiago, Chile\\%
    $^3$Instituto de Ingeniería Matemática y Computacional, Universidad Católica de Chile\\
    $^4$Institute of Fundamental Technological Research, Polish Academy of Sciences\\%
}

\begin{document}

\maketitle

\begin{abstract}
     Delle Rose et al.~(COLT'23)  introduced an effective version of the Vapnik-Chervonenkis dimension, and showed that it characterizes improper PAC learning with total computable learners. In this paper, we introduce and study a similar effectivization of the notion of Littlestone dimension.  Finite effective Littlestone dimension is a necessary condition for computable online learning but is not a sufficient one---which we already establish for classes of the effective Littlestone dimension 2. However, the effective Littlestone dimension equals the optimal mistake bound for computable learners in two special cases: a) for classes of Littlestone dimension 1 and b) when the learner receives as additional information an upper bound on the numbers to be guessed. Interestingly, finite effective Littlestone dimension also guarantees that the class consists only of computable functions. 
\end{abstract}

\section{Introduction}
Two fundamental models of machine learning, PAC learning and online learning, have been recently revisited from the viewpoint of computability theory~\cite{agarwal2020learnability,sterkenburg2022characterizations,delle2023find,hasrati2023computable}. In the classical setting, a learning algorithm is understood as a function, getting a sample $S$ and an input $x$ and outputting its prediction of the value on $x$. Although this is called an ``algorithm'', it is not assumed to have a Turing machine that computes it.
The existence of a learning algorithm for a hypothesis class can 
be characterized by a combinatorial dimension of that class, namely, the VC dimension in the case of PAC learning and the Littlestone dimension in the case of online learning.

What if we do require a learning algorithm to be computable by a Turing machine? We obtain ``computable counterparts'' of  PAC and online learning models that might no longer be characterized just by a combinatorial dimension. 
For instance, Strekenburg~\cite{sterkenburg2022characterizations} constructs a class with finite VC dimension, given by a decidable set of functions with finite support, that has no computable PAC learner, even if the learner is allowed to be improper (output functions outside the class). Likewise, Hasrati and Ben-David~\cite{hasrati2023computable} observe that there is a class that has Littlestone dimension 1 and consists of finitely supported functions but does not have an online learner, computable by a partial Turing machine  (the machine must be defined on realizable samples but otherwise it might not halt)   with finite number of mistakes. 

Delle Rose at al.~\cite{delle2023find} recently characterized computable PAC learning via an  \emph{effectivization} of the notion of the VC dimension. The usual VC dimension of a class $H$   is defined as the maximal size of a subset of the domain where functions from $H$ can realize all dichotomies.  The ``dual'' defintion is  the minimal $d$ such that for any subset of size $d+1$ there exists a dichotomy, not realizable by $H$. In the effective version of VC dimension, there must be a Turing machine that, given an $(d+1)$-size subset, outputs a dichotomy, not realizable by $H$. The minimal $d$ for which such a Turing machine exists is called the effective VC dimension of $H$. As Delle Rose et al.~\cite{delle2023find} show, classes admitting a computable PAC learner are exactly classes having finite effective VC dimension. They asume that the learner can be improper but has to be computed by a \emph{total} Turing machine, that is, it must halt even on non-realizable inputs.

In this paper, we introduce a similar ``effectivization'' of the Littlestone dimension and study its relationship with the computable online learning. The usual Littlestone dimension of a hypothesis class $H$ is defined as the maximal $d$ for which there exists a depth-$d$ Littlestone tree with every branch realizable by $H$. Following the idea of~\cite{delle2023find}, we define the effective Littlestone dimension of $H$ as the \emph{minimal} $d$ for which there exists a Turing machine that, given a Littlestone tree of depth $d+1$, indicates a branch, not realizable by $H$.

Our contribution with respect to the effective Littlestone dimension consists of the following.
\begin{itemize}
 \item In a similar manner, we define the notion of the effective threshold dimension and observe that classes with finite effective Littlestone dimension coincide with classes of finite effective threshold dimension.
    \item We observe that a class that admits an online learner, computable by a total Turing machine (that we all, for brevity, a total computable online learner), that makes at most $d$ mistakes, has effective Littlestone dimension at most $d$.
   
    \item We show that the converse does not hold. We construct a class of effective Littlestone dimension 2 that does not admit even a partial computable online learner (``partial'' means that the Turing machine, computing it, might not halt on some non-realizable samples) with a finite number of mistakes. 
    \item On the positive side, we show that effective Littlestone dimension is equivalent to a computable online learning ``with an upper bound''. In this setting, the learner is given in an advance an upper bound on numbers it will see in the game.


    \item We also show that every class of finite effective Littlestone dimension consists of computable functions. As a consequence, every class of effective Littlestone dimension 1 admits a total computable online learner with 1 mistake.
\end{itemize}
Similar failure of the combinatorial characterization of computable learning was recently observed by Gourdeau, Tosca, and Urner~\cite{gourdeau2024computability} for computable robust PAC learning.


\section{Preliminaries}
By \emph{hypothesis classes} we mean sets of functions from $\mathbb{N}$ to $\{0,1\}$. By \emph{samples} we mean finite sequences of pairs from $\mathbb{N}\times \{0, 1\}$. A sample $S = (x_1, y_1)\ldots (x_k, y_k)$ is \emph{consistent} with a function $f\colon\mathbb{N}\to\{0,1\}$ if $f(x_1) = y_1, \ldots, f(x_k) = y_k$. A sample $S = (x_1, y_1) \ldots (x_k, y_k)$ is \emph{realizable} by a hypothesis class $H$ (or $H$-realizable, for brevity) if there is a function in $H$ with which $S$ is consistent.

A \emph{learner} is a partial function $L\colon (\mathbb{N}\times\{0, 1\})^*\times\mathbb{N}\to\{0,1\}$ (thus, the first input to $L$ is a sample and the second input is a natural number).  We say that a learner $L$ is a learner for a hypothesis class $H$ if for every $H$-realizable sample $S$ the value $L(S, x)$ is defined for every $x\in\mathbb{N}$. A total learner is a learner which is defined everywhere. Sometimes we write ``partial learner'' to stress that a statement holds not only for total learners.

A learner $L$ is computable if there exists a Turing machine that outputs $L(S, x)$ on $(S, x)$ for which $L$ is defined, and does not halt on $(S, x)$ for which $L$ is not defined.

For a given sample $S$, the learner induces a (possibly, partial) function $L_S\colon\mathbb{N}\to\{0,1\}$ by setting $L_S(x) = L(S, x)$, to which we refer as the hypothesis of $L$ after the sample $S$.

Let $L$ be a learner and $S = (x_1, y_1)\ldots (x_k, y_k)$ be a sample. \emph{The number of mistakes} of $L$ on $S$ is the number of $i\in\{1, \ldots, k \}$ such that 
$L((x_1, y_1)\ldots (x_{i-1}, y_{i-1}), x_{i}) \neq y_{i}$. One can interpret this quantity as follows. Imagine that $L$ receives pairs of $S$ one by one. Each pair  $(x_{i}, y_{i})$ is given like this: first $L$ receives $x_{i}$ and is asked to predict $y_{i}$, using its knowledge of the preceding pairs in the sample. After $L$ makes a prediction, the true value of $y_{i}$ is revealed, causing a \emph{mistake} if the prediction differs from $y_{i}$.

A learner $L$ for a hypothesis class $H$ is called an \emph{online learner} for $H$ \emph{with at most} $d$ \emph{mistakes} if $L$ makes at most $d$ mistakes on any $H$-realizable sample.
\begin{lem}
\label{lemma_cons}
    Let $H$ be a hypothesis class and $L$ be an online learner for $H$ with at most $d$ mistakes, for some $d\in\mathbb{N}$. Then every function $f\in H$ coincides with $L_S$ for some sample $S$, consistent with $f$.
\end{lem}
\begin{proof}
    Indeed, if there is no such sample, we can construct a sample, consistent with $f$, on which $L$ makes more than $d$ mistakes. Namely, we start with the hypothesis of $L$ after the empty sample. It disagrees with $f$ on some $x_1\in \mathbb{N}$ which we put to the sample as $(x_1, f(x_1))$, causing the first mistake. The hypothesis of $L$ after $(x_1, f(x_1))$ disagrees with $f$ on some $x_2$, and we add this $(x_2, f(x_2))$ to the sample, forcing the second mistake, and so on. In this way, we can force arbitrarily many mistakes.   
\end{proof}
\begin{cor}
Let $H$ be a hypothesis class that for some $d\in\mathbb{N}$, has a computable online learner $L$, making at most $d$ mistakes on $H$. Then all functions in $H$ are computable.
\end{cor}
\begin{proof}
By Lemma \ref{lemma_cons}, every function $f\in H$ coincides with $L_S$ for some sample $H$-realizable sample $S$. Since $L$ is a learner for $H$ and is computable, the function $L_S$ is computable. 
\end{proof}

By a \emph{Littlestone tree} of depth $d$ we mean a complete rooted binary tree of depth $d$ where: (a) edges are directed from parents to children, with each edge labeled by 0 or 1 such that every non-leaf node has one out-going 0-edge and one out-going 1-edge; and (b) non-leaf nodes are labeled by natural numbers. Every edge in such a tree can be assigned a pair $(x, y)\in\mathbb{N}\times \{0,1\}$ where $x$ is the natural number,  labelling  node this edge starts at, and $y$ is the bit, labelling this edge.
Thus, every directed path in this tree can be assigned a sample, obtained by concatenating pairs, assigned to its edges. Now, for a vertex $v$ of a Littlestone tree $T$, and for a hypothesis class $H$, we say that $v$ is $H$-realizable if the sample, written on the path from the root of $T$ to $v$, is $H$-realizable.

The \emph{Littlestone dimension} of a class $H$, denoted by $\ldim(H)$, is the minimal $d\ge 0$ such that in every $(d+1)$-depth Littlestone tree $T$ there exists a leaf which is not $H$-realizable. The \emph{effective Littlestone dimension} of a class $H$, denoted by $\eldim(H)$,
 is the minimal $d\ge 0$ for which there exists a total Turing machine that, given as input 
a Littlestone tree of depth $d+1$, outputs some leaf of this tree which is not $H$-realizable.

\begin{prop}[\cite{littlestone1988learning}]
\label{prop_littlestone}
For any class $H$, the minimal $d\ge 0$ for which there exists an online learner for $H$ with at most $d$ mistakes is equal to $\ldim(H)$.

\end{prop}

If $H$ is a hypothesis class, then for $x\in \mathbb{N}$ and $b\in\{0, 1\}$, by $H^x_b$ we denote the class $\{f\in H\mid f(x) = b\}$.

\begin{prop}[\cite{littlestone1988learning}]
\label{prop_less}
    For any hypothesis class $H$ of finite positive Littlestone dimension, and for every $x\in \mathbb{N}$, either $H^x_0$ or $H^x_1$ have smaller Littlestone dimension than $H$. 
\end{prop}




For a sample $S$, a \emph{cylinder}, induced by $S$, is the set of functions  $f\colon\mathbb{N}\to\{0,1\}$, consistent with $S$. Unions of cylinders induce on $\{0,1\}^\mathbb{N}$ a topology, homeomorphic to the Cantor space. Cylinders are clopen in this topology. We use a well-known fact that the Cantor space is compact.

A subset of $\{0,1\}^\mathbb{N}$ is \emph{effectively open} if it is a union of an enumerable set of cylinders. A  subset of $\{0,1\}^\mathbb{N}$ is \emph{effectively closed} if the complement to it is effectively open.

\begin{prop}
\label{prop_enum}
Let $X\subseteq \{0,1\}^\mathbb{N}$ be effectively open. Then the set of cylinders $C$ that are subsets of $X$ is enumerable.
\end{prop}
\begin{proof}
There exists a computable enumeration $\{C_n\}_{n = 1}^\infty$ of cylinders such that $X = \bigcup_{n = 1}^\infty C_n$. We enumerate all cylinders $C$  for which there exists $N\in\mathbb{N}$ such that $C\subseteq  \bigcup_{n = 1}^N C_n$. By compactness, since every cylinder is closed, in this way we will enumerate all cylinders $C$ such that $C\subseteq \bigcup_{n = 1}^\infty C_n$.
\end{proof}

\section{Effective threshold dimension}
\label{sec_threshold}

Let $t \in \mathbb{N}$ and $(x_1, \ldots, x_t)\in\mathbb{N}^t$ be a sequence of $t$ natural numbers. For $i = 1,\ldots, t$, the \emph{$i$th threshold}  on $(x_1, \ldots, x_t)$ is a sample:
\[(x_1,0)\ldots (x_{i-1}, 0) (x_i,1)\ldots (x_t, 1).\]

.
The \emph{threshold dimension} of a hypothesis class $H$, denoted by $\tdim(H)$,
is the largest natural number $t$ for which there exists a sequence $(x_1, \ldots, x_t)\in\mathbb{N}^{t}$ such that for all $i = 1, \ldots, t$, the $i$th threshold on $(x_1, \ldots, x_t)$ is $H$-realizable.

Shelah~\cite{Shelah} have shown that a class has finite Lilttlestone dimension if and only it has finite threshold dimension.  Hodges~\cite{hodges1997shorter} and Alon et al.~\cite{Alon-et-al} have shown the following quantitive version of the Shelah's result.
\begin{thm}[\cite{hodges1997shorter, Alon-et-al}] \label{thm:classical-thresholds}
\label{thm_shelah}
     For any hypothesis class $H$, we have: 
    \begin{enumerate}
        \item $\lfloor\log_2\ldim(H)\rfloor \le\tdim(H)$;
        \item $\lfloor\log_2\tdim(H)\rfloor \le\ldim(H)$.
    \end{enumerate}
\end{thm}

We demonstrate that a hypothesis class $H$ has finite effective Littlestone dimension if and only if it has finite \emph{effective threshold dimension}. Here, the effective threshold dimension of $H$, denoted by $\etdim(H)$ is the minimal $t \ge 0$ for which there exists a total Turing machine $w$ which, having on input a sequence $(x_1,\ldots, x_{t+1})\in\mathbb{N}^{t + 1}$, outputs some $i\in\{1, \ldots, t + 1\}$ such that  the $i$th threshold on $(x_1, \ldots, x_{t + 1})$ is not $H$-realizable.

In fact, we show that any upper bound on the Littlestone dimension by the threshold dimension, and vice versa, extends to the effective versions of these dimensions. More precisely, the following theorem holds.

\begin{thm}
\label{thm_effective_shelah}
    \begin{itemize}
        \item  (a) for $d\in\mathbb{N}$, let $t_d$ denote the maximal possible threshold dimension of a hypothesis class with Littlestone dimension at most $d$; then any hypothesis class $H$ with effective Littlestone dimension $d$ has effective threshold dimension at most $t_d$.
        \item (b) for $t\in\mathbb{N}$, let $d_t$ denote the maximal possible Littlestone dimension of a hypothesis class with threshold dimension at most $t$; then any hypothesis class $H$ with effective threshold dimension $t$ has effective threshold dimension at most $d_t$ 
    \end{itemize}
\end{thm}
\begin{proof}
Let us show \emph{(a)}. Let $H$ be a class of effective Littlestone dimension $d$. We show that its effective threshold Littlestone dimension is at most $t_d$. We have a total turing machine $A$ that, given a Littlestone tree $T$ of depth $d +1$, outputs a leaf of $T$ which is not $H$-realizable. We convert $A$ into a Turing machine $w$ that, given a sequence $\overline{x} = (x_1,\ldots, x_{t_d + 1})\in\mathbb{N}^{t_d + 1}$, outputs some $i \in\{1, \ldots, t_d + 1\}$ such that the $i$th threshold on $\overline{x}$  is not $H$-realizable.

The machine $w$ starts by calculating a list $T_1,\ldots, T_m$ of all Littlestone trees of depth $d +1$ where node labels are taken from the set $\{x_1,\ldots, x_{t_d + 1}\}$. Then the machine runs $A$ on all of these Littlestone trees. Each time $A$ outputs a leaf in some of these trees, it writes down the sample, written on the path to this leaf. Let $S_1, \ldots, S_m$ be the resulting list of samples. Observe that, by definition of $A$, all these samples are not $H$-realizable. The machine proceeds by constructing a set $\widehat{H}$  of functions $g\colon\{x_1, \ldots, x_{t_d + 1}\}\to\{0,1\}$ such that $g$ is not consistent with $S_\ell$ for all $\ell = 1, \ldots, m$. That set includes all restrictions of $f\in H$ to the set $\{x_1, \ldots, x_{t_d + 1}\}$. Finally, the machine goes through all $i = 1,\ldots, t_{d}+1$, checking, whether the $i$th threshold on  $(x_1,\ldots, x_{t_d + 1})$ is  $\widehat{H}$-realizable,  brute-forcing all functions in $\widehat{H}$. Whenever it finds  a not $\widehat{H}$-realizable threshold, which is also automatically not $H$-realizable, the machine outputs the corresponding $i$.

Such threshold exists because the threshold dimension of $\widehat{H}$, as of an hypothesis class over the domain $\{x_1, \ldots, x_{t_d}\}$, is at  most $t_d$. Indeed, its 
  Littlestone dimension is at most $d$  because in every Littlestone tree $T_\ell$ over the domain $\{x_1, \ldots, x_{t_d}\}$, there exists a leaf with the sample $S_\ell$ which is  not $\widehat{H}$-realizable. And by definition, the threshold dimension of a class with  Littlestone dimension at most $d$ cannot exceed $t_d$.

The statement \emph{(b)} is proved similarly. Now we have to convert a Turing machine $w$ which, given a sequence $\overline{x} = (x_1,\ldots, x_{t +1})\in\mathbb{N}^{t + 1}$, outputs a threshold on $\overline{x}$ which is not $H$-realizable, into a Turing machine $A$ that, given a depth-$(d_t + 1)$ Littlestone tree $T$, outputs a leaf in $T$ which is not $H$-realizable. We let $D_T$ be the set of all natural numbers, appearing in $T$. We go through all $(t +1)$-length sequences, consisting of numbers from $D_T$, run $w$ on them and construct the list of all threshold that it outputs (that are all not $H$-realizable). We then construct a  set $\widehat{H}$ of functions $g\colon D_T\to\{0, 1\}$ that are inconsistent with all these thresholds, and this set includes all restrictions of function from $H$ to $D_T$. We notice that, by construction, the threshold dimension of $H$ is at most $t$, which means that its Littlestone dimension is at most $d_t$. We use this to find in  $T$, which is a tree of depth $d_t +1$, a leaf which is not $\widehat{H}$-realizable. This leaf is automatically not $H$-realizable.
\end{proof}
\begin{cor} \label{cor:efff-thresholds}
     For any hypothesis class $H$, we have: 
    \begin{enumerate}
        \item $\lfloor\log_2\eldim(H)\rfloor \le\etdim(H)$;
        \item $\lfloor\log_2\etdim(H)\rfloor \le\eldim(H)$.
    \end{enumerate}
\end{cor}

\section{Effective Littlestone dimension vs.~computable online learning}
\label{sec_learning}
\begin{prop}
\label{prop_total_implies}
For any hypothesis class $H$ and for any $d$, we have the following. If $H$
admits a total computable online learner which makes at most $d$ mistakes, then the effective Littlestone dimension of $H$ is at most $d$.
\end{prop}
\begin{proof}
Let $L$ be a total computable online learner for $H$ with at most $d$ mistakes.
Given a $(d+1)$-depth Littlestone tree $T$, we find a leaf of it on which $L$ makes $d+1$ mistakes. Namely, we give $L$ the number from the root, wait for its prediction (since $L$ is total, we will receive it), go to the child which contradicts this prediction, give the number from this child, and so on. The sample on the path to this leaf cannot by $H$-realizable because $L$ makes at most $d$ mistakes on $H$-realizable samples.
\end{proof}

Main result of this section is that the converse of this proposition is false already for $d = 2$ (although, as we will see later, it is true for $d = 1$).
\begin{thm}
\label{thm_counterexample}
    There exists a class $H$ of effective Littlestone dimension 2 which, for all $d$, does not have a partial computable online learner with at most $d$ mistakes.
\end{thm}
\begin{proof}




In our construction, to make sure that $H$ has effective Littlestone dimension at most 2, we establish two things: (a) $H$ has ordinary Littlestone dimension at most 2, (b) $H$ is effectively closed.


Why do (a) and (b) imply that $H$ has effective Littlestone dimension at most 2? We have to provide an algorithm that, given a depth-3 Littlestone tree $T$, gives a leaf $\ell$ of $T$ which  is not $H$-realizable. Such leaf $\ell$ exists because, by (a), the ordinary Littlestone dimension of $H$ is at most 2. Out task is to find it. For a leaf $\ell$,   let $S_\ell$ be the sample, written on the path to $\ell$. Let $C_\ell$ be the cylinder, induced by $S_\ell$. A leaf $\ell$ is not $H$-realizable if and only if $C_\ell$ is a subset of the complement to $H$. Since, $H$ is effectively closed by (b), the complement to it is effectively open. Hence, by Proposition \ref{prop_enum}, the set of cylinders that are subsets of the complement to $H$ is enumerable. We  start enumerating them until $C_\ell$ for some leaf $\ell$ of $T$ appears in this enumeration.

In our construction, we ensure effective closeness of $H$ by defining it  via an enumerable set of ``local restriction''. Each local restriction is of the form ``at this (finite) set of positions, you cannot have this combination of values''.  Thus, each local restriction is, formally, a complement to a cylinder. The class $H$ will consist of functions, satisfying all these restrictions. In other words, $H$ will be an intersection of an enumerable set of complements to the cylinders. Hence, the complement to $H$ will be a union of the corresponding  enumerable set of cylinders, as required in the definition of an effectively closed set.


Fix a computable enumeration $L_1, L_2, L_3, \ldots$ of all partial computable learners. We say that a class $H$ ``fools'' a learner $L_i$ if either
\begin{itemize}
\item (a) there is an $H$-realizable sample $S$ and $x\in\mathbb{N}$ such that $L_i(S, x)$ does not halt; 
\end{itemize}
or
\begin{itemize}
\item (b) there is a function $f\in H$ and an infinite sequence of natural numbers $\{x_n\in\mathbb{N}\}_{n = 1}^\infty$ such that, denoting $S_n = (x_1, f(x_1))\ldots (x_n, f(x_n))$, we have  that $L_i(S_{n - 1}, x_n)$ is defined but differs from $f(x_n)$ for every $n\ge 1$. In other words, $L_i$ incorrectly predicts $f$ every time on the sequence $(x_1, x_2, x_3, \ldots)$.
\end{itemize}
If $H$ fools $L_i$, then there is no $d\in\mathbb{N}$ for which  $L_i$ is an online learner for $H$ with at most $d$ mistakes. Indeed, in the  case of (a), $L_i$ is not a learner for $H$, and in  the case of (b), for every $n$ there exists an $H$-realizable sample $S_n$ on which $L_i$ makes $n$ mistakes.
We will construct $H$ that fools every $L_i$ but has Littlestone dimension 2 and is effectively closed.

\medskip

Let us start with a simpler task -- for every $i$,  we construct an effectively closed class $\widehat{H}_i$ that fools $L_i$ (but maybe not other partial computable learners). The class  $\widehat{H}_i$ will have at most 2 functions. The construction works in  (potentially infinitely many) iteration.

In the first iteration, we give $1$ to $L_i$ for prediction, on the empty sample, that is, we start computing $L_i(\text{empty}, 1)$. In parallel, we  start listing restrictions of the form ``$f(1) = f(k)$'' for $k =2, 3,$ and so on (forbidding different values at $1$ and $k$). If $L_i$ never halts, we will list all such restrictions. As the result, $\widehat{H}_i$ will consist of two constant functions. In this case, $L_i$ is fooled by not halting for some $\widehat{H}_i$-realizable sample, namely, for the empty one.

Assume now that $L_i(\text{empty}, 1)$ halts, outputting  $p_1\in \{0,1\}$. Up to this moment, we have listed 
    restrictions ``$f(1) = f(k)$'' for $k$ up to some $k_1\in\mathbb{N}$. We now add a restriction, forbidding $f(1)$ to be $p_1$. In other words, we set $f(1) = b_1 = \lnot p_1$. With this, the first iteration ends. So far, we have achieved two things. First, functions, satisfying our current restrictions, are exactly functions with $f(1) = \ldots = f(k_1) = b_1$. Second, $L_i(\text{empty}, 1)$ halts but its output is different from $b_1$.

More generally, in our construction, after $n$  iterations, for some $k_1, \ldots, k_n\ge 1$ and for some $b_1,\ldots, b_n\in\{0, 1\}$, the following requirements will be fulfilled:
\begin{itemize}
\item Setting $x_1 = 1,\,\, x_2 = x_1 + k_1, \ldots,\,\, x_n = x_{n - 1} + k_{n - 1}$ and $S_m = (x_1, b_1), \ldots, (x_m, b_m)$ for $m = 0, \ldots, n$, we have that $L_i(S_{m - 1}, x_m)$ halts but with the output, different from $b_m$, for every $m = 1,\ldots, n$.
\item functions, satisfying our current list of restrictions, are precisely functions $f$, satisfying:
\begin{align}
\label{eq_intervals}
\begin{split}
f(x_1) = &\ldots = f(x_ 1 + k_1 - 1) = b_1,\\
f(x_2) = &\ldots = f(x_ 2 + k_2 - 1) = b_2,\\
&\vdots\\
f(x_n) = &\ldots = f(x_ n + k_n - 1) = b_n.
\end{split}
\end{align}
\end{itemize}
Assuming these conditions are fulfilled after $n$ iterations, we show how to fulfil them after $n + 1$ iterations. We start computing $L_i(S_n, x_{n+1})$  for $x_{n + 1} = x_n + k_n$. In parallel, we start listing restrictions of the form ``$f(x_{n + 1}) = f(x_{n + 1} + k - 1)$'' for $k \ge 2$. If $L_i(S_n, x_{n+1})$ never halts, we will list all such restrictions. As the result, the class $\widehat{H}_i$ will consist of two functions that are defined by  \eqref{eq_intervals} on  numbers less than $x_{n+1}$ and are constant on $\{x_{n+1}, x_{n + 1} +1, x_{n  +1} + 2, \ldots\}$. Both these function are consistent with $S_n$ as $S_n$ is a part of \eqref{eq_intervals}; this fools $L_i$ as it does not halt on $(S_n, x_{n+1})$ while $S_n$ is  $\widehat{H}_i$-realizable.

Assume now that $L_i(S_n, x_{n+1})$ halts, outputting $p_{n+1}\in\{0,1\}$. Up to this point, we have listed ``$f(x_{n + 1}) = f(x_{n + 1} + k - 1)$'' restrictions for $k$ up to some $k_{n+1}$. We add a restriction ``$f(x_{n + 1}) = b_{n+1}= \lnot p_{n+1}$'' and end with this the $(n+1)$st iteration. This adds a line $f(x_{n+1}) = \ldots = f(x_{n+1} + k_{n+1} - 1) = b_{n+1}$ to \eqref{eq_intervals}, as required. Moreover, by making sure that  $L_i(S_n, x_{n+1})$ halts but outputs  $\lnot b_{n+1}$, we extend the requirement $L_i(S_{m-1}, x_m) \neq b_m$ to $m = n + 1$.

As is already observed, if some iteration never ends, $L_i$ will be fooled by not halting on some input with an $\widehat{H}_i$-realizable sample. Now, imagine that every iteration ends after finitely many steps. Then \eqref{eq_intervals} will be true for all $n$, leaving in $\widehat{H}_i$ a single function $f$, which is equal to $b_1$ on $[x_1, x_2)$, to $b_2$ on $[x_2, x_3)$, and so on. This $f\in \widehat{H}_i$, together with the sequence $(x_1, x_2, x_3,\ldots)$, will fool $L_i$. 

\medskip

We now give a single effectively closed class $H$ of Littlestone dimension at most $2$ that fools every $L_i$. We partition natural numbers into infinitely many infinite disjoint blocks in some computable way, assigning each $L_i$ one of the blocks.
We will have two kind of restrictions. First, for every pair of numbers from different blocks, we will forbid both of them having value 1, forcing every function in $H$ to have value 1 in at most one of the blocks. Restrictions of the second type will involve only numbers from the same block. Thus, for every $i$, the will be the ``$i$th block restrictions'', and their union over $i  = 1, 2, 3,...$ will be the set of second-type restrictions.  

For every $i$, we list the $i$th block restrictions in a way that
fools $L_i$ as in the construction of the class $\widehat{H}_i$, but using the set of numbers of the $i$th block instead of $\{1, 2, 3, \ldots,\}$. Additionally, we do it with one modification. As a result of this modification, the class $\widehat{H}_i$ will potentially have 3 functions. Namely,  the all-0 function will be added to the class $\widehat{H}_i$ if it was not there already.  

In more detail, 
every restriction
that we have for $L_i$,
saying ``you cannot have these values in these positions'', is turned into infinitely many restrictions, where for every $x$ from the $i$-th block, we say ``you cannot have these values in these positions and have $1$ at position $x$ simultaneously''. Any function, satisfying old restrictions, satisfies all these new restrictions because new restrictions are weaker. However, no new function, apart from the all-0 function, can be added to $\widehat{H}_i$ in this way. Indeed, any function $f$ with at least one value 1, violating some old restriction, will violate a new restriction where as $x$ we take some number on which $f$ is equal to 1. As a result, there will be at most 2 functions in $H$ that have value 1 on some number from the $i$th block.

We  ``almost'' established that $H$ fools every $L_i$. Namely, there will be a function $f_i$, defined on the $i$th block and satisfying the $i$th block restrictions, that fools $L_i$. That is, either there will be an input with a sample, consistent with this function, on which $L_i$ does not halt, or there will be an infinite sequence of numbers from the $i$th block on which $L_i$ predicts values of $f_i$ incorrectly every time. It remains to argue that $f_i$ can be extended to a function $f_i\colon\mathbb{N}\to\{0,1\}$  satisfying other restrictions, defining $H$ -- first-type restrictions and restrictions for other blocks. Namely, set $f_i(x)=0$ for all $x$ outside the $i$th block. This works because  all these other restrictions involve at least one label 1 for a number outside the $i$th block (this is why we had to modify the construction of $\widehat{H}_i$, including at least one label 1 to every restriction!).

It remains to show that the Littlestone dimension of $H$ is at most $2$. Due to the first-type restrictions, no function in $H$ can have value 1 in two distinct blocks.  Thus, $H$ can be presented as
\[H = \{\text{all-0 function}\} \cup H_1 \cup H_2 \cup H_3\ldots,\]
where $H_i$ iss the set of functions from $H$ that have value 1 on some number from the $i$th block. As we have noted, the size of every $H_i$ is at most 2.  Hence, there is an online learner for $H$ with at most 2 mistakes, implying by Proposition \ref{prop_littlestone} that $\ldim(H) \le 2$. This learner first predicts 0 on every number. If it is wrong, it is because there is a positive label in some block. This leaves the algorithm with at most 2 possible functions left. The learner first predicts according to one of them, and, in case of the second mistake, according to the second one.

\end{proof} 

\section{Equivalence in the bounded regime}
On the positive side, we consider a  modification of online learning where the learner initially gets an upper bound $N$ on the numbers it will receive for prediction. It can be arbitrarily large, but the bound on the number of mistakes $d$ should not depend on $N$. We call it \emph{online learning in the bounded regime}. 
We show that effective Littlestone dimension characterizes
computable learnability in this setting. As a corollary, we get the separation between computable online learning in the bounded and the unbounded regime. For some class, online learning with bounded number mistakes is possible when the learner gets an arbitrary bound on the numbers, but not possible without a bound.

To be precise, a  learner with an upper bound is a, possibly partial, function $L \colon\mathbb{N}\times (\mathbb{N}\times\{0,1\})^*\times \mathbb{N}\to\{0,1\}$ (compared to normal learners, it has one more input, an upper bound). AS before, we say that $L$ is a learner for a hypothesis class $H$ if $L(N,S, x)$ is defined for every $N,x\in\mathbb{N}$ and for every $H$-realizable sample $S$.

Let $L$ be a learner with an upper bound for $H$. 
We say that $L$ \emph{online learns} $H$ \emph{in the bounded regime with at most} $d$ \emph{mistakes} if, for any $N\in\mathbb{N}$ and any $H$-realizable sample $S=(x_1,y_1), \dots, (x_k,y_k)$ with the property that $x_1, \dots, x_k \le N$, there are at most $d$ numbers $i \in \{1, \dots, k\}$ such that 
$L(N,(x_1,y_1), \dots, (x_{i-1}, y_{i-1}), x_{i}) \ne y_{i}.$
\begin{prop}
    A hypothesis class $H$ has effective Littlestone dimension at most $d$ if and only if there is a total computable learner with an upper bound which online learns $H$ in the bounded regime with at most $d$ mistakes.
\end{prop}
\begin{proof}
    Assume that $L$ is a computable a learner with an upper bound  which online learns $H$ in the bounded regime with at most $d$ mistakes. We show that $\eldim(H) \le d$.  Let $T$ be a Littlestone tree of depth $d+1$ where we have to output a leaf which is not  $H$-realizable. We take as $N$ the largest number, appearing in $T$, and run the same procedure as in the proof of Proposition \ref{prop_total_implies}, with $L$ having $N$ as the additional input.

    Next, assume that $H$ has effective Littlestone dimension at most $d$. Hence, there is an algorithm $A$ that, given a $(d+1)$-depth Littlestone tree $T$, outputs a leaf which is not $H$-realizable.
    We construct a learner $L$ that,  given an upper bound $N$, goes through all Littlestone trees of depth $d+1$ with node labels at most $N$, computes all samples that are indicated by $A$ in these trees, and finds the set $H_N$
    of
    all functions on the first $N$ natural numbers that are inconsistent with all these samples. The set $H_N$ includes all functions that can be continued to a function in $H$. On the other hand, the Littlestone dimension $H_N$ is at most $d$ as ``witnessed'' by $A$. The class $H_N$ is  over a finite domain, and we have a complete description of it, so we find an online learner with at most $d$ mistakes for it by the brute-force.      
\end{proof}
One could wonder whether the similar equivalence could be proven for finite classes of functions but exchanging "computable" for some of form of time-bounded computability, such as 'polynomial-time computable'. We leave this as an interesting direction for further research.

\section{Effective Littlestone dimension and computability}
\begin{thm}
\label{thm_computable}
Let $H$ be a hypothesis class with finite effective Littlestone dimension. Then all functions in $H$ are computable.
\end{thm}
 \begin{proof}
We establish the theorem by induction on $\eldim(H)$. When $\eldim(H) = 0$, we have an algorithm that, given a depth-1 Littlestone tree $T$, outputs a leaf of $T$ which is not $H$-realizable. In other words, for a given number $x\in\mathbb{N}$, written in the root of $T$, it indicates $b\in\{0,1\}$ such that $f(x) \neq b$ for all $f\in H$. By outputting $\lnot b$ on $x$ we obtain a program for the unique function $f\in H$.


For the induction step, we need the following lemma, which is an analog of the Proposition \ref{prop_less} for effective Littlestone dimension.
\begin{lem}
\label{lemma_less}
    For any class $H$ of finite positive effective Littlestone dimension, and for any $x\in\mathbb{N}$, either $H^x_0$ or $H^x_1$ have smaller effective Littlestone dimension that $H$. 
\end{lem}
\begin{proof}
Let $d = \eldim(H) > 0$.
    There exists an algorithm $A$ that, given a $(d+1)$-depth Littlestone tree, outputs a leaf of it which is not $H$-realizable.

    We now describe two algorithms, $A_0$ and $A_1$, and show that either $A_0$ establishes that $\eldim(H_0^x)\le d - 1$, or $A_1$ establishes that $\eldim(H_1^x) \le d - 1$. Namely, both algorithms receive on input a $d$-depth Littlestone tree (here we need a condition $d > 0$ so that the notion of ``$d$-depth trees'' makes sense). The algorithm $A_0$ is supposed to output a leaf which is not $H^x_0$-realizable. Likewise, $A_1$ is supposed to output a leaf which is not $H_x^1$-realizable. 

The algorithm $A_0$ works as follows. Let its input be a depth-$d$  Littlestone tree $T_0$. The algorithm goes over all depth-$d$ Littlestone trees $T_1$, and for each of them, does the following. It constructs a tree
 $T = (x, T_0, T_1)$, where the root is labeled by $x$, the $0$-subtree coincides with $T_0$, and the $1$-substree coincides with $T_1$. The algorithm gives this $T$ to $A$. If  $A$ outputs a leaf in the $0$-subtree of $T$, that is, inside $T_0$, the algorithm $A_0$ outputs this leaf as its answer, and halts. Otherwise, $A_0$ proceeds to the next $T_1$. 

If $A_0$ ever halts on $T_0$, then the leaf $\ell$ of $T_0$ that it outputs is not $H_0^x$-realizable. Indeed, consider the sample $S_\ell^0$, written on the path from the root of $T_0$ to $\ell$. Assume for contradiction that this sample is $H_0^x$-realizable. Then the sample $(x,0)S_\ell^0$ is $H$-realizable.  But $\ell$ is the output of $A$ on $T$, and $(x,0)S_\ell^0$ is  written on the path from the root of $T$ to $\ell$, a contradiction.


 The problem with $A_0$ is that it might not halt on some $T_0$. This happens when, for all $T_1$, the algorithm $A$ on input $T = (x, T_0, T_1)$ outputs a leaf in $T_1$. We now define the algorithm $A_1$. It  receives a depth-$d$ Littlestone tree $T_1$ on input (where it supposed to indicate a not $H^x_1$-realizable leaf), and runs $A$ on all trees of the form $(x, T_0, T_1)$, waiting until $A$ indicates a leaf in the $1$-subtree. By the same argument, whenever $A_1$ halts, its output is correct.

The only case when both algorithms fail is when there exist depth-$d$ Littlestone trees $T_0^\prime, T_1^\prime$ such that $A_0$ does not halt on $T_0^\prime$ and $A_1$ does not halt on $T_1^\prime$. This means that
$A$ goes to the $1$-subtree in all trees of the form $(x, T_0^\prime, T_1)$, and goes to the $0$-subtree in all trees of the form $(x, T_0, T_1^\prime)$. However, this means that $A$ does not output anything in the tree $(x, T_0^\prime, T_1^\prime)$, a contradiction.
\end{proof}

    Let us now finish the induction step. Assume we have a class $H$ of effective Littlestone dimension $d> 0$, and for all smaller value of effective Littlestone dimension, the theorem is already proved.

Without loss of generality, we may assume that $H$ is effectively closed. Indeed, take the algorithm $A$ that, given a $(d+1)$-depth Littlestone tree $T$, outputs a leaf of $H$ which is not $H$-realizable. Say that a function $f\colon\mathbb{N}\to\{0,1\}$ agrees with $A$ if there is no depth-$(d + 1)$ Littlestone tree $T$ on which $A$ outputs a leaf which is consistent with  $f$. All functions in $H$ agree with $A$. Consider the class $\widehat{H}\supseteq H$ of all functions that agree with $A$. The effective Littlestone dimension of $\widehat{H}$ is at most $d$ as established by the algrotihm $A$. In turn,  $\widehat{H}$ is effectively closed. Indeed, the complement to it consists of all function that are consistent with at least one leaf that $A$ outputs on  depth-$(d + 1)$ Littlestone trees. To enumerate the set of cylinders whose union is the complement to $\widehat{H}$, we go though all  depth-$(d + 1)$ Littlestone trees $T$, compute the leaf $\ell = A(T)$, and add the cylinder $C_\ell$, induced by this leaf, to the enumeration.

From now on, we assume that the class $H$ is effectively closed. 
   By Lemma \ref{lemma_less}, for every $x\in\mathbb{N}$, either $\eldim(H_0^x) < d$ or $\eldim(H_1^x) < d$. 
Assume first that for some $x\in\mathbb{N}$, we have $\eldim(H_0^x) < d$ or $\eldim(H_1^x) < d$. Then by the induction hypothesis, both $H_0^x$ and $H_1^x$ consist of computable functions. It remains to notice that $H =  H_0^x\cup H_1^x$.

Assume now that for every $x\in\mathbb{N}$, either $\eldim(H_0^x) = d$ or $\eldim(H_1^x) = d$. Consider the function $f\colon\mathbb{N}\to\{0, 1\}$, defined by $\eldim(H_{f(x)}^x) = d$ for every $x$. By Lemma \ref{lemma_less}, we have $\eldim(H_{\lnot f(x)}^x) < d$ for every $x\in\mathbb{N}$. Hence, any function $g\in H$, different from $f$, is computable, as it belongs to $H_{\lnot f(x)}^x$ for some $x\in\mathbb{N}$.  It remains to show that if $f\in H$, then it is computable.

First, consider the case when there exists a sample $S$ which is consistent with $f$ but not with any other $g\in H$. We describe an algorithm that, given $x\in\mathbb{N}$, computes $f(x)$. Observe that the sample $S (x, \lnot f(x))$  is not $H$-realizable because it is inconsistent with $f$ and $S$ is inconsistent with all the other functions in $H$. At the same time, $S (x,  f(x))$ is $H$-realizable because it is consistent with $f\in H$. Since $H$ is effectively closed, the complement to it is effectively open. By Proposition \ref{prop_enum}, the set of cylinders that are subsets of the complement to $H$, is enumerable. Hence, the set of samples that are not $H$-realizable, is enumerable. We enumerate this set until a sample of the form $S(x, y)$ for some $y\in\{0, 1\}$ appears.  Since $S (x, \lnot f(x))$  is not $H$-realizable and $S(x, f(x))$ is, we have that $y = \lnot f(x)$. We output $\lnot y = f(x)$.


 Assume now that for every sample $S$, consistent with $f$, there exists $g \in H$, different from $f$, which is also consistent with $S$. 
We  use an online learning algorithm with ``consistent oracle''~\cite{kozachinskiy2024simple, assos2023online}. A consistent oracle for a class $H$ is a mapping that, given an $H$-realizable sample $S$, outputs a function $f_S\in H$, consistent with this sample. More precisely, it gives an oracle access to $f_S$, meaning that given $S$ and $x\in\mathbb{N}$, it allows to evaluate $f_S(x)$. Kozachinskiy and Steifer\cite{kozachinskiy2024simple} constructed an algorithm that, for any class $H$ of Littlestone dimension $d$, given only access to a consistent oracle for $H$, online learns it with at most $O(256^d)$ mistakes.  This algorithm, to compute $L(S, x)$, the prediction on $x$ after the sample $S$, uses consistent oracle only for $S$ and its subsamples, making sure that it never applied to a non-$H$-realizable sample.
 
 We get back to the class $H$ in question, and we take any consistent oracle $H$ that never uses function $f$. Such oracle exists because for any sample, consistent with $f$, there exists another function in $H$, consistent with this sample.  We know that all function in $H$, apart from $H$, are computable. Hence, this consistent oracle uses only computable functions.
 
 We consider the online learner $L$ of Kozachinskiy and Steifer~\cite{kozachinskiy2024simple}, equipped with this consistent oracle. By Lemma \ref{lemma_cons}, there exists a sample $S$, consistent with $f$, such that $L_S$ coincides with $f$, that is, $L(S, x) = f(x)$ for all $x\in \mathbb{N}$. We  show that $L_S$ is computable. Indeed, in its computation, the consistency oracle is queried only for finitely many samples. It is enough to hardwire programs for the output functions of the consistency oracle on these samples.



 \end{proof}
\begin{cor}
    Let $H$ be a class of effective Littlestone dimension 1. Then it has a total computable online learner with at most 1 mistake.
\end{cor}
\begin{proof}
   Assume first that $H$ is finite. By Theorem \ref{thm_computable}, all finitely many functions of $H$ are computable. In this case, we can realize the standard optimal algorithm of Littletone~\cite{littlestone1988learning} by a total Turing machine. In case when $\ldim(H) = 1$, it works like this: given $x\in \mathbb{N}$, it takes $b\in\{0, 1\}$ such that $\ldim(H^x_b) = 0$ (existing by Proposition \ref{prop_less})and predicts $\lnot b$ so that when it is wrong, we are in $H^x_b$ where there is exactly one function. To realize this algorithm by a total Turing machine, we need to be able to decide, whether a sample is realizable, and whether it is realizable by exactly one function from $H$. We can achieve this by evaluating all functions from $H$ on the numbers from the sample.
   
From now on we assume that $H$ is infinite. We may also assume that $H$ is effectively closed, by the same argument as in the proof of Theorem \ref{thm_computable}.

First, observe that there is no $x\in\mathbb{N}$ such that $\ldim(H^x_0) = \ldim(H^x_1) = 0$ because otherwise $H =H^x_0 \cup H^x_1$ has size at most 2. Therefore, we can define a function $f\colon\mathbb{N}\to\{0, 1\}$ by setting $f(x)$ such that $\ldim(H^x_{f(x)}) = 1$. We claim that this function belongs to $H$. Indeed, if not, since $H$ is closed, some  sample 
\[S = (x_1, f(x_1))\ldots(x_k, f(x_k))\]
is consistent with $f$ but not $H$-realizable. But then $H = H^{x_1}_{\lnot f(x_1)} \cup \ldots \cup H^{x_k}_{\lnot f(x_k)}$. By the definition of $f$, we have $\ldim(H^{x_i}_{\lnot f(x_i)}) = 0$ for every $i = 1,\ldots, k$. Hence, in $H$ there are at most $k$ function, so it cannot be infinite. 

Therefore, $f\in H$ and hence is computable by Theorem \ref{thm_computable}. We give a total computable online learner $L$ for $H$ with at most 1 mistake, working as follows. Given a sample $S$ and $x\in\mathbb{N}$, we define $L(S, x) = f(x)$ if $S$ is consistent with $f$ (this can be checked computably since $f$ is computable). If $S$ is not consistent with $f$, we start enumerating samples that are not $H$-realizable, using Proposition \ref{prop_enum} and the fact that $H$ is effectively closed. Whenever a sample of the form $S(x,y)$ for some $y\in\{0, 1\}$ appears, we output $L(S, x) = \lnot y$. 

This learner makes at most 1 mistake on $H$-realizable sample. Namely, it can make a mistake only when the first pair, inconsistent with $f$, appears. In turn,
this learner is total. Indeed, whenever $S$ is not consistent with $f$, there is a pair of the form $(x,\lnot f(x))$ in $S$ for some $x\in\mathbb{N}$. It is impossible for both $S (x, 0)$ and $S(x, 1)$ to be $H$-realizable because $H_{\lnot f(x)}^x$ has at most 1 function.

\end{proof}

\end{document}